\documentclass[11pt]{article}
\usepackage{geometry}
\usepackage{color}
\usepackage[utf8]{inputenc} 
\usepackage[T1]{fontenc}    
\usepackage{hyperref}       
\usepackage{url}            
\usepackage{booktabs}       
\usepackage{amsfonts}       
\usepackage{nicefrac}       
\usepackage{microtype}      

\usepackage{amsmath,amsfonts,amsthm,amssymb, graphicx, setspace, verbatim}
\usepackage{algorithm, algorithmic}
\usepackage{subfig}
\usepackage{color}
\usepackage{geometry}
\usepackage{graphics}
\usepackage{caption}
\usepackage{enumitem}
\usepackage{float}
\usepackage{dsfont}

\newtheorem{theorem}{Theorem}
\newtheorem{assumption}{Assumption}

\theoremstyle{example}
\newtheorem{lemma}{Lemma}

\newtheorem{remark}{Remark}
\theoremstyle{remark}

\def\P{{\mathbb P}}
\def\E{{\mathbb E}}
\def\R{{\mathbb R}}
\def\Z{{\mathbb Z}}
\def\G{{\mathcal G}}
\def\M{{\mathcal M}}
\def\S{{\mbox{sign}}}
\def\T{{ \mathrm{\scriptscriptstyle T} }}

\title{On Constructing Confidence Region for Model Parameters in Stochastic Gradient Descent via Batch Means}

\author{%
	Yi Zhu \\
	Northwestern University\\
	Jing Dong\\ 
	Columbia Business School\\
}

\date{}

\begin{document}
	\maketitle		
\begin{abstract}
		In this paper, we study a simple algorithm to construct asymptotically valid confidence regions for
		model parameters using the batch means method. 
		The main idea is to cancel out the covariance matrix which is hard/costly to estimate.
		In the process of developing the algorithm, we establish process-level functional central limit theorem for Polyak-Ruppert averaging based stochastic gradient descent estimators.
		We also extend the batch means method to accommodate more general batch size specifications.
	\end{abstract}
	
	\section{Introduction}
	
	Stochastic Gradient Descent and variants of it have been widely used in model-parameter estimation for either online learning or when data sizes are very large
	\cite{RobbinMonros:1951,polyak1992acceleration,Kingma:2015}. As the estimators we construct via stochastic gradient descent is random, it is desirable to be able to quantify the estimation errors incurred.
	While there is a rich literature studying convergence rate of the objective function or the parameter estimation errors based on stochastic gradient descent, e.g., \cite{zhang2004solving,agarwal2009information,Nemirovski:2009}, 
	much less is known about the statistical inference for true model parameters 
	\cite{Glynn:2002, Toulis:2017, Fang:2017,Chen:2018, Su:2018}. 
	Following the later line of work, in this paper, we propose a simple procedure to construct asymptotically valid confidence regions for model parameters
	based on a cancellation method known as the batch means. 
	The confidence region we construct is tight as it takes the covariance structure of the parameters into account.
	
	We consider the classic setting where the model parameters, $x^*$, can be characterized as the minimizer of a convex objective function, which is also known as the loss function. Specifically,
	\begin{equation}\label{eq:opt}
	x^*=\arg\min \left(H(x):=E[h(x,\zeta)]\right),
	\end{equation}
	where $h$ is a real-valued function, $x$ is a $d$-dimensional parameter, and $\zeta$ is a $d^{\prime}$-dimensional random vector.
	Stochastic gradient descent is an iterative algorithm to solve \eqref{eq:opt}. In its simplest form, the $t$-th iteration takes the form
	\[X_t=X_{t-1} - \gamma_t \nabla_x h(X_{t-1},\zeta_t),\]
	where $\nabla_x h$ is the gradient of $h$ with respect to $x$ and $\gamma_t$ is the step size.
	If we take $\bar X_t= t^{-1}\sum_{i=0}^{t-1}X_i$ as an estimator for $x^*$, then under certain regularity conditions, \cite{polyak1992acceleration} establish that
	\[t^{1/2}\left(\bar X_t - x^*\right)\Rightarrow N(0, \Sigma) \mbox{ as $t\rightarrow\infty$},\]
	where $\Rightarrow$ denotes convergence in distribution,  $N(0,\Sigma)$ denotes a Gaussian random vector with mean $0$ and covariance matrix $\Sigma$, and
	\[\Sigma=\nabla^2 H(x^*)^{-1} U \nabla^2 H(x^*)^{-1},\]
	where $\nabla^2 H(x^*)$ is the Hessian of $H$ at $x^*$, 
	and $U=E[\nabla_x h(x^*, \zeta) \nabla_x h(x^*, \zeta)^T]$. 
	If we know the value of $\Sigma$, then a natural way to construct the $95\%$ confidence region for $x^*$ is
	\[\hat R_t=\{x\in \R^d: t(\bar X_t - x)^T\Sigma^{-1}(\bar X_t-x)\leq \chi_{d,0.05}^2\},\]
	where $\chi_{d,0.05}^2$ is the $95\%$-quantile of the chi-squared distribution with $d$ degrees of freedom.
	The confidence region is asymptotically valid in the sense that
	$\lim_{t\rightarrow\infty} \textnormal{pr}(x^*\in \hat R_t)=0.95.$
	The main challenge here is that covariance matrix $\Sigma$ is unknown and it is very costly to construct consistent estimators of $\Sigma$ \cite{Chen:2018}.
	
	To address the challenge, we introduce a cancellation method, called the batch means, from the stochastic simulation literature \cite{schruben1983confidence,Glynn:1990,MunozG01}. 
	The main idea is to construct the statistics in a special way to cancel out the unknown covariance matrix.
	The method was introduced to deal with steady-state estimation problems, where we use the time average of the stochastic process as an estimator of the steady-state mean.
	Despite the elegance of the method, existing results in the literature do not allow us to apply it directly in the stochastic gradient descent setting.
	This is because in steady-state estimation problems, the stochastic process is time-homogeneous, while in stochastic gradient descent, if we view $\{X_t:t\geq 0\}$ as a stochastic process, the transition kernel is time-dependent due to the decreasing step sizes. 
	
	The main contribution of this paper is that we rigorously establish the validity of the batch means method in the stochastic gradient descent setting.
	This provides us with a simple way to construct asymptotically valid confidence regions for model parameters.
	The method utilizes the output of the stochastic gradient descent algorithm itself, and it does not require any modification to the underlying algorithm.
	We also extend the batch means method to allow more general batch size specifications and provide some guidance on how to select the batch sizes.
	Our analysis relies on the process-level convergence result for $\{\bar X_t:t\geq0\}$, which is stronger than the 
	large sample convergence result established in the literature.

	\section{Batch means method} \label{sec:main}
		Consider the case where $H(x)$ is strongly convex with a unique minimizer at $x^*$.
		We follow the Polyak-Ruppert averaging iteration, 
		\begin{equation}\label{eq:rec}
		X_t=X_{t-1}-\gamma_t \G(X_{t-1},\zeta_t),
		\end{equation}
		where $E[\G(X_{t-1},\zeta_t)|X_{t-1}]=\nabla H(X_{t-1})$ and $\gamma_t=at^{-r}$ for some $a>0$ and $r\in(1/2,1)$.
		The batch means method divides the stochastic gradient descent sample path $\{X_t: 0\leq t\leq T\}$ into $m$ non-overlapping batches, where the $i$-th batch is of size $b_i:=\lceil Tw_i\rceil$.
		Here, $m\in \Z_+$ with $m>d$, and $w=(w_1,\dots,w_m)\in \R_+^{m}$, where $\Z_+$ is the set of positive integers and $\R_+$ is the set of positive real numbers. $m$ and $w$ are the parameters for the batch means method.
		The method is asymptotically valid for a wide range of parameter specifications. As for the pre-limit performance, we will discuss how to ``fine-tune" these parameters in Section \ref{sec:size}.
		We define $\tau_i = \sum_{j=1}^{i}b_i$. 
		Then the $i$th batch contains $\{X_{\tau_{i-1}+1}, \dots, X_{\tau_i}\}$ and its batch mean is defined as 
		$\Xi_i=b_i^{-1}\sum_{t=\tau_{i-1}+1}^{\tau_i}X_t$.
		
		The basic idea of the batch means method is that for $T$ large enough, $\Xi_i$'s are approximately independent 
		$N(x^*, (1/b_i)\Sigma)$.
		Then we can construct $F$ type of statistics based on the $m$ batch means. In particular, we consider the statistics
		\begin{equation}\label{eq:stat}
		\Gamma_T = m(m-d)(d(m-1))^{-1}(\bar X_T - x^*)^TS_m^{-1}(T)(\bar X_T - x^*)
		\end{equation}
		where
		$\bar X_T:=T^{-1}\sum_{t=1}^{T}X_t ~~~\mbox{ and } ~~~ S_m(T):=(m-1)^{-1}\sum_{i=1}^{m}(\Xi_i-\bar X_T)(\Xi_i-\bar X_T)^T$
		How $\Gamma_T$ works will be made precise in Theorem \ref{th:main}.
		The actual procedure to construct the confidence region is summarized in the following Algorithm.
		
	
	\begin{algorithm}
		\caption{Construct a $100(1-\delta)\%$ confidence region for $x^{*}$} \label{alg:main}
		\begin{algorithmic}[1]
			\STATE \textbf{Input:} The SGD sample path $\{X_t: 0\leq t\leq T\}$, the number of batches $m$, the relative batch length parameter $w$
			\STATE  Find the appropriate scaling parameter $\alpha_m(\delta,w)$. 
			\STATE  Calculate the batch means $\Xi_i$ for $i=1,2,\dots,m$
			\STATE  Calculate $\bar X_T$ and $S_m(T)$
			\STATE \textbf{Output:} $R_T=\left\{x \in \mathbb{R}^d:  \frac{m(m-d)}{d(m-1)}(\bar X_T -x)^T S_m^{-1}(T)(\bar X_T -x) \leq \alpha_m(\delta,w)\right\}.$
		\end{algorithmic} 
	\end{algorithm}
	
	
		The confidence regime constructed in Algorithm 1 is asymptotic valid in the sense that if 
		the scaling parameter $\alpha_m(\delta,w)$ is properly chosen, then
		$\lim_{T\rightarrow\infty}\textnormal{pr}(x^*\in R_T)=1-\delta.$
		The key now is to calibrate the appropriate scaling parameter $\alpha_m(\delta,w)$.
		The value of $\alpha_m(\delta,w)$ is determined by the asymptotic behavior of $\Gamma_T$.
		Theorem \ref{th:main} characterizes the limiting distribution of $\Gamma_T$ as $T\rightarrow\infty$, and is the main result of this paper.
		Before we present the theorem, we first introduce a few assumptions, which are standard for the convergence analysis of Polyak-Ruppert averaging, e.g., \cite{Chen:2018, polyak1992acceleration}.
		We define $\Delta_t := X_t - x^*$, and $\xi_t=(\xi_t(1), \dots, \xi_t(d))$ as 
		\begin{equation} \label{eq:md}
		\xi_t:=\G(X_{t-1},\zeta_t)-\nabla H(X_{t-1}). 
		\end{equation} 
		
		\begin{assumption}\label{ass:main1}
			$H(x)$ is continuously differentiable and strongly convex with parameter $C>0$, i.e., for any $x$ and $y$
			$H(y)\geq H(x)-\nabla H(x)^T(y-x)+\frac{C}{2}\|y-x\|_2^2.$
			$\nabla H(x)$ is Lipschitz continuous with parameter $L>0$, i.e., for any $x$ and $y$,
			$\|\nabla H(x) - \nabla H(y)\|_2 \leq L\|x-y\|_2,$
			and $\nabla^2 H(x^*)$ exists.
		\end{assumption}
		
		\begin{assumption}\label{ass:main2}
			$(\xi_t: t \geq 1)$ is a martingale-difference process with respect to the filtration $\mathcal{F}=\{\mathcal{F}\}_{t\geq 1}$ generated by $(\zeta_t: t \geq 1)$, and it satisfies the following:\\
			\noindent{\bf 1.} The conditional covariance of $\xi_t$ has an expansion around $x=x^{*}$:
			$E[\xi_t \xi_t^T|\mathcal{F}_{t-1}] = U + r(\Delta_{t-1}),$
			for some positive definite matrix $U$, and there exit constants $S_1>0$ and $S_2>0$, such that for any $\Delta\in \mathbb{R}^d$,
			$ \|r(\Delta)\|_2 \leq S_1 \|\Delta\|_2 + S_2 \|\Delta\|^2_2 $.\\
			\noindent{\bf 2.} There exists $M\in(0,\infty)$, such that $\|\xi_t\| \leq M$ almost surely, $\forall t\geq 1$.
		\end{assumption}
		
		\begin{remark}\label{rem:main1}
			Assumption \ref{ass:main1} ensures the convergence of $\bar X_t$ to a unique global optimal $x^*$ \cite{polyak1992acceleration}. 
			Assumption \ref{ass:main2} provides sufficient conditions to establish the functional Central Limit Theorem (FCLT) for partial sums of $\xi_t$'s. 
		\end{remark}
		

		We define the function $g_m: C^d[0,1] \times \mathbb{R}^{m} \rightarrow \mathbb{R}^{d \times d}$ as
		\[g_m(x,w) =\frac{1}{m-1} \sum_{i=1}^m \left(\frac{x(c_i)-x(c_{i-1})}{c_i- c_{i-1}} -x(1)\right)\left(\frac{x(c_i)-x(c_{i-1})}{c_i- c_{i-1}} -x(1)\right)^T,\]
		where $c_0=0$ and $c_i=c_{i-1}+w_i$. We are now ready to introduce the main theorem.
		
		\begin{theorem} \label{th:main}
			Under Assumption \ref{ass:main1} and \ref{ass:main2}, for $\Gamma_T$ defined in \eqref{eq:stat}, when $m>d$ and $w>0$,
			\[ \Gamma_T \Rightarrow m(m-d)(d(m-1))^{-1} Z^{T} g_m(B, w)^{-1}Z \mbox{ as $T\rightarrow\infty$,}\] 
			where $Z$ is a standard $d$-dimensional Gaussian random vector, $B$ is standard $d$-dimensional Brownian motion, and $Z$ is independent of $g_m(B, w)$. Furthermore, if we set $\alpha_m(\delta,w)$ as the $(1-\delta)$-quantile of $m(m-d)(d(m-1))^{-1} Z^{T} g_m(B, w)^{-1}Z$, then
			\[\lim_{T\rightarrow\infty} \textnormal{pr} (x^*\in R_T)=1-\delta.\]
		\end{theorem}
		
		

		We note from Theorem \ref{th:main} that the scaling parameter $\alpha_m(\delta,w)$ does not depend on the underline problem instances. It only depends on the batch means parameters $m$ and $w$. 
		In the special case of evenly-split batch size, i.e., $w_i=1/m$,
		\[m(m-d)(d(m-1))^{-1} Z^{T} g_m(B, w)^{-1}Z\]
		follows an $F$ distribution with $d$ and $m-d$ degrees of freedom.
		We will discuss a different splitting scheme in Section \ref{sec:size} and provide the corresponding scaling parameter table (Table \ref{tab:main0}).
		
	\section{Selection of the batch means parameters} \label{sec:size}
	The confidence region constructed using the batch means method is asymptotic valid regardless of our choice of $m$ and $w$, as long as $m>d$ and $w>0$.
	However, different $m$ and $w$ will affect the pre-limit performance of the procedure. In this section, we study how to choose the parameters for the batch means method.
	The analysis is divided into two parts. We first study for a fixed $m$, how to choose the batch sizes $w$. We then study how to choose $m$.
	
	The pre-limit performance is essentially determined by how close the distribution of
	\[\left((b_1/\sqrt{T})(\Xi_1 - x^*), \dots ,(b_m/\sqrt{T})(\Xi_m - x^*)\right)\] 
	is to 
	$\left(G(B(c_1)-B(c_0)), \dots, G(B(c_m)- B(c_{m-1}))\right)$.

	\subsection{Batch size}
	Note that the pre-limit $\Xi_i$'s are correlated while the limiting $(B(c_i)-B(c_{i-1}))$'s are uncorrelated.
	Thus, one important quantity we want to minimize is the correlation between $\Xi_i$ and $\Xi_{i+1}$.
	
	To understand the correlation between $\Xi_{i}$ and $\Xi_{i+1}$, we follow the arguments in \cite{Chen:2018}. 
	We first note that for $t$ large, $X_t$ is close to $x^*$.
	Thus,
	\[ \nabla H(X_{t-1}) \approx \nabla H(x^*) + \nabla ^2 H(x^*) (X_{t-1} - x^*) = A \Delta _{t-1}, \] 
	where $ A :=\nabla^2 H(x^*)$ and $\Delta_t = X_t - x^*$, and the equality follows as $\nabla H(x^*)=0$.
	Then by the recursion formula \eqref{eq:rec}, we have
	\[ \Delta_t \approx (I -\gamma_t A) \Delta_{t-1} +\gamma_t \xi_t, \] 
	where $I$ is the identity matrix and $\xi_t$ is defined in \eqref{eq:md}.
	This further indicates that for $i$ and $j$ large, the correlation between $ \Delta_i$  and $\Delta_j$ is approximately
	\[\prod ^{j-1}_{t=i} || I - \gamma_{t} A|| \approx \exp\left(-\lambda(A) \sum^{j-1}_{t=i} \gamma_{t}\right),\]
	where $\lambda(A)$ denote the smallest eigenvalue of $A$.
	With the goal of balancing the correlation between $\Xi_i$ and $\Xi_{i+1}$, 
	we can choose $w$ according to
	\[\min_{w}\max_i\exp\left(-\lambda(A)\sum_{t=\tau_{i-1}}^{\tau_i} \gamma_t\right).\]
	It is easy to see that the minimum is achieved when
	$\sum_{t=\tau_{i-1} +1}^{\tau_i} \gamma_t$'s are equal. 
	In this case, we can set
	\[\tau_i=\left(i/m\right)^{1/(1-r)}T.\]
	Note that for this specification of $\tau_i$'s, we have increasing batch sizes, i.e., $w_i$'s are increasing in $i$.
	This is similar to the batch size splitting rule proposed in \cite{Chen:2018}.
	For what follows, we shall refer to this specification as the ``increasing batch size" (IBS) allocation.
	The main difference between our method and the one in \cite{Chen:2018} is that
	the method in \cite{Chen:2018} requires sending $m$ to infinity as $T$ goes to infinity, 
	while our method holds $m$ fixed. 
	We will conduct more comparisons of the two methods in Section \ref{sec:sim}.
	
	Table \ref{tab:main0} provides some of the commonly used scaling parameters for IBS
	with different values of $d$ and $m$. As these quantiles are estimated using Monte Carlo simulation, we also provide 
	the corresponding $95\%$ confidence intervals.
	\begin{table}[htbp]
		\caption{$95\%$-quantile of $\frac{m(m-d)}{d(m-1)}Z^Th_m(B,w)^{-1}Z$ with IBS allocation} \label{tab:main0}
		\centering
		\begin{tabular}{c|c|c|c|c}\hline
			$d$ & $1$ & $2$ & $3$ & $4$    \\ \hline
			$m=10$  & $2.93\pm0.01$ & $2.92\pm0.01$ & $3.13\pm0.01$ & $3.50\pm0.01$  \\ \hline
			$m=20$ & $2.18\pm 0.01$ & $2.00\pm0.01$ & $1.95\pm0.01$ & $1.97\pm0.01$ \\ \hline
			$m=30$  & $1.91\pm0.01$ & $1.71\pm0.01$ & $1.64\pm0.01$ & $1.62\pm0.01$   \\ \hline
			$m=40$  & $1.76\pm0.01$ & $1.55\pm0.01$ & $1.47\pm0.01$ & $1.50\pm0.01$  \\ \hline
			$m=60$  & $1.58\pm0.01$ &$1.38\pm0.01$ &$1.29\pm0.01$ & $1.26\pm 0.01$  \\ \hline
			$m=100$  & $1.42\pm 0.01$&$1.21\pm0.01$ &$1.13\pm0.01$ & $1.09\pm0.01$  \\ \hline
			$m=120$  &$1.33\pm 0.01$ &$1.15\pm0.01$ &$1.07\pm0.01$ & $1.03\pm 0.01$ \\ \hline
			$m=150$  & $1.28\pm 0.01$& $1.09\pm 0.01$ & $1.01\pm 0.01$ & $0.97\pm0.01$ \\ \hline\hline
			$d$ & $10$ & $50$ & $80$ & $100$    \\ \hline
			$m=10$  & NA  & NA  & NA & NA  \\ \hline
			$m=20$  & $2.46\pm 0.01$ & NA & NA & NA \\ \hline
			$m=30$  & $1.74\pm 0.01$ & NA & NA & NA   \\ \hline
			$m=40$  & $1.48\pm 0.01$  & NA & NA & NA  \\ \hline
			$m=60$  & $1.24\pm 0.01$ & $2.49\pm 0.04$ & NA & NA \\ \hline
			$m=100$  & $1.02\pm 0.01$& $1.31 \pm 0.01$ &$1.81\pm 0.01$ & NA \\ \hline
			$m=120$  & $0.97\pm 0.01$& $1.18 \pm 0.01$ & $1.43 \pm 0.01$ & $1.81\pm 0.01 $\\ \hline
			$m=150$  & $0.91\pm 0.01$  & $1.07\pm 0.01$  & $1.22\pm 0.01$ & $1.35\pm 0.01 $  \\ \hline
		\end{tabular}
	\end{table}
	
	We next show some numerical experiments about different choices of batch sizes.
	We compare three different specifications: i) IBS, ii) even splitting (ES), 
	and iii) decreasing batch size (DBS) where we reverse the batch size specification of IBS.
	Table \ref{tab:main1} summarizes results. 
	
	For Table \ref{tab:main1} and subsequent numerical experiments, we focus on 
	two classes of examples: linear regression and logistic regression. 
	For linear regression, we write $b_i=x^{*T} a_i + \epsilon_i$ where $\epsilon_i$'s are iid $N(0,1)$.
	In this case, $\zeta=(a,b)$ and $h(x,\zeta)=(b-x^Ta)^2$. 
	For logistic regression, we consider $b_i\in\{-1,1\}$ with
	$\P(b_i=1|a_i)=(1+\exp(-x^{*T}a_i))^{-1}$. In this case $\zeta=(a,b)$ and 
	$h(x,\zeta)=\log(1+\exp(-b x^{T} a))$.
	When not specified,  the  true parameters $x^{*}$ is a $d$-dimensional vector linearly spaced between
	$0$ and $1$. We set the baseline number of iterations at $n:=10^5$. In all the examples, our goal is to achieve $95\%$ coverage rate.
	The estimated coverage rate is based on 1000 independent replications of the procedure. We also report the corresponding $95\%$
	confidence interval for the coverage rate.
	
	We observe from Table \ref{tab:main1} that as the number of iteration increases, 
	all three batch size specifications are approaching the correct coverage rate, $0.95$. 
	For a relatively small number of iterations, IBS and ES achieve a higher coverage rate
	than DBS. 
	
	\begin{table}[htbp]
		\caption{Coverage rate comparison for different batch size specifications} \label{tab:main1}
		\centering
		\begin{tabular}{c|c|c|c|c}\hline
			& $n$ & $4n$ & $7n$ & $10n$  \\ \hline\hline
			\multicolumn{5}{c}{Linear regression with $d=2$}\\ \hline
			IBS  & $0.975\pm0.009$ & $0.955\pm0.013$ & $0.970\pm0.010$ & $0.971\pm0.009$  \\ \hline
			ES & $0.938\pm 0.015$ & $0.947\pm0.014$ & $0.951\pm0.013$ & $0.950\pm0.013$  \\ \hline
			DBS  & $0.787\pm0.025$ & $0.878\pm0.020$ & $0.909\pm0.017$ & $0.912\pm0.019$  \\ \hline\hline
			\multicolumn{5}{c}{Logistic regression with $d=2$}\\ \hline
			IBS  & $0.934\pm0.015$ & $0.932\pm0.015$ & $0.946\pm0.014$ & $0.948\pm0.013$  \\ \hline
			ES  & $0.899\pm 0.018$ & $0.917\pm0.018$ & $0.934\pm0.015$ & $0.933\pm0.015$  \\ \hline
			DBS  & $0.842\pm0.023$ & $0.908\pm0.017$ & $0.932\pm0.018$ & $0.930\pm0.015$  \\ \hline
		\end{tabular}
	\end{table}
	

	\subsection{Number of batches}
	
	We next look into different choices of $m$ for $m\geq d+1$. We divide the analysis into two parts. 
	We first analyze the limiting volume of the confidence region for different choices of $m$. 
	We then analyze the pre-limit performance.
	
	The volume of the confidence region, which is a $d$-dimensional ellipsoid, takes the form
	\[V_d(m,w):=\left( \frac{d(m-1)}{m(m-d)}\right)^{d/2}\det(S_m(T)^{1/2})\alpha_m(\delta,w)^{d/2}q_d,\]
	where $q_d=\pi^{d/2}/\Gamma(d/2+1)$, with $\Gamma$ denoting the Gamma function, is the volume of a $d$-dimensional unit sphere.
	From Theorem \ref{th:fcll}, we have 
	$$\det((TS_m(T))^{1/2}) \Rightarrow \det(G)^2\det(h_m(B,w)^{1/2})
	\mbox{ as $T\rightarrow\infty$}.$$
	Thus, in Figure \ref{fig:m}, we compare 
	\[v_{d}(m,w) := \left(\frac{d(m-1)}{m(m-d)}\right)^{d/2}\E[\det(h_m(B,w)^{1/2})]\alpha_m(\delta,w)^{d/2}\]
	for different values of $m$. 
	We observe that as $m$ increases, the volume of the confidence region decreases. 
	However, there is a diminishing effect of increasing $m$ on decreasing the volume.
	Moreover, for pre-limit, the larger $m$ is, the smaller the size of each batch would be, 
	which implies that the batch means are further from their corresponding asymptotic distributions.
	These suggest that $m$ should not be too large. This is especially important when $T$ is relatively small.
	
	\begin{figure}[htb]
		\caption{Compare $v_{d}(m,w)$ for different values of $m$ and $d$.} 
		\label{fig:m}\centering
		\subfloat[$d=1$]{
			\includegraphics[width=0.3\textwidth]{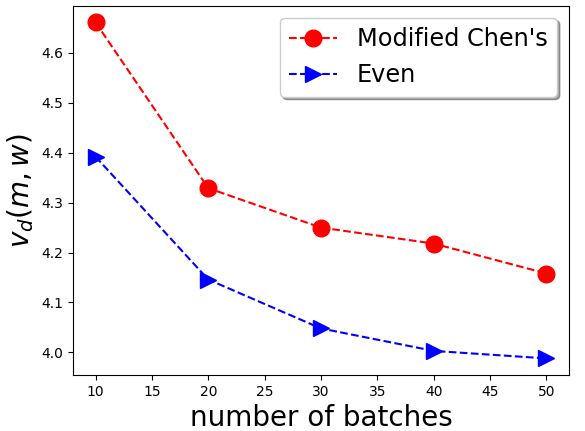}
		} ~~~
		\subfloat[$d=2$]{
			\includegraphics[width=0.3\textwidth]{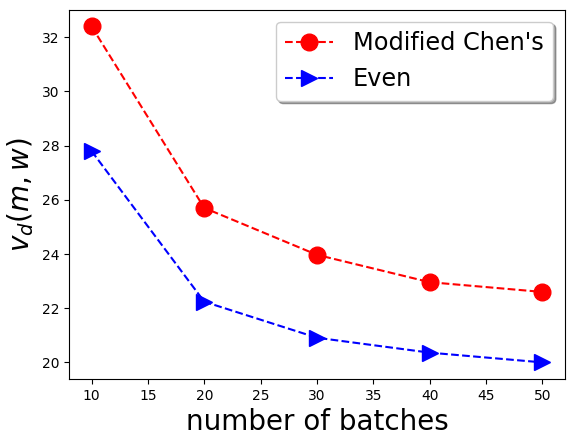}
		}~~~
		\subfloat[$d=5$]{
			\includegraphics[width=0.3\textwidth]{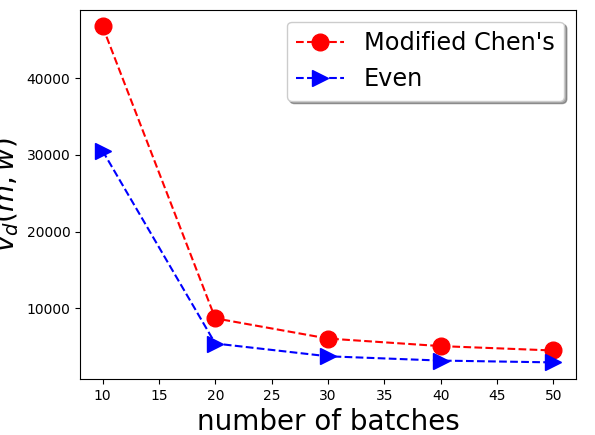}
		}
	\end{figure}
	
	In Table \ref{tab:main5}, we compare the pre-limit performance for different values of $m$. We use IBS for the batch size specification. 
	We focus on a relatively small number of iterations in these examples and we observe that when the numbers of iterations are small, large values of $m$ can lead to substantial under-coverage.
	
	\begin{table}[htbp]
		\caption{Coverage comparison for different values of $m$, logistic regression with $d=3$} \label{tab:main5}
		\centering
		\begin{tabular}{c|c|c|c|c}\hline
			& $0.1n$ & $0.4n$ & $0.7n$ & $n$  \\ \hline
			$m=10$  & $0.913\pm0.017$ & $0.933\pm0.015$ & $0.947\pm0.013$ & $0.933\pm0.015$  \\ \hline
			$m=20$  & $0.814\pm 0.024$ & $0.897\pm0.018$ & $0.919\pm0.017$ & $0.927\pm0.016$  \\ \hline
			$m=30$  & $0.730\pm0.027$ & $0.876\pm0.020$ & $0.909\pm0.017$ & $0.906\pm0.018$  \\ \hline
			$m=40$  & $0.615\pm0.030$ & $0.817\pm0.024$ & $0.845\pm0.022$ & $0.883\pm0.019$  \\ \hline
		\end{tabular}
	\end{table}
	
	\section{Comparison to other methods} \label{sec:sim}
	In this section, we compare our batch means method to two recently developed methods to draw statistical inference for model parameters in SGD.
	Specifically, the methods are developed in \cite{Chen:2018} and \cite{Su:2018}, which we refer to as batch means with an increasing number of batches (BMI) and 
	hierarchical incremental gradient descent (HiGrad), respectively. 
	We also introduce a fourth method, which is known as the sectioning method \cite{Glynn:2002}. This method is similar to the batch means method, but instead of dividing
	a single sample path into $m$ batches, we generate $m$ independent sample path of equal length. This method can also be viewed as a special case of HiGrad where the number of levels is $1$.
	
	BMI is mainly designed to draw marginal inference, i.e., it constructs confidence intervals for each parameter (dimension) separately. Thus, it does not impose $m\geq d+1$. 
	However, we note from Lemma \ref{lm:h} that when $m\leq d$, the estimated covariance matrix $S_m(T)$ is likely to be degenerate. Indeed, Figure \ref{fig:cov} plots the histogram
	of the determinant of $S_m(T)$ for a logistic regression problem with $n$ iterations. Note that in this case, BMI suggests setting $m=\lceil n^{0.25} \rceil=18$. 
	We compare two cases, one has $d=10<m$, the other has $d=20>m$. We observe that when $d>m$, the determinant of $S_m(T)$ is concentrated around zero.
	
	\begin{figure}[htb]
		\caption{Histogram for the determinant of $S_m(T)$ with $m=18$}
		\label{fig:cov}\centering
		\subfloat[$d=10$]{
			\includegraphics[width=0.36\textwidth]{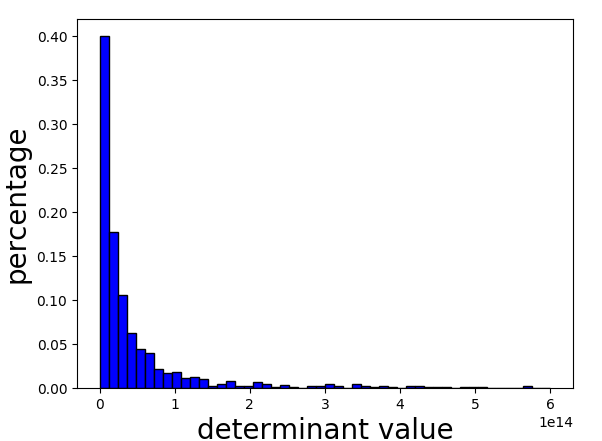}
		} ~~~
		\subfloat[$d=20$]{
			\includegraphics[width=0.36\textwidth]{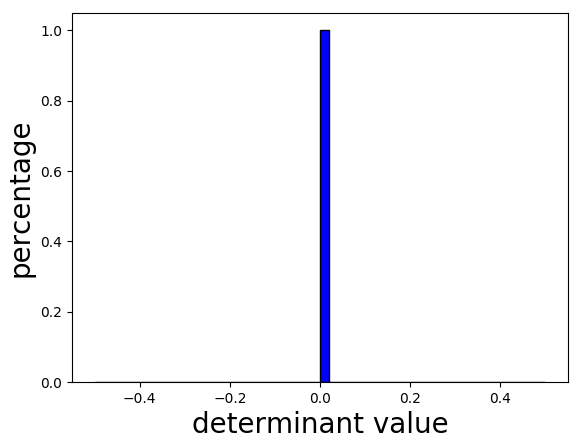}
		}~~~
	\end{figure}
	
	HiGrad has versions for both marginal inference and joint inference. However, we note that there are a lot more parameters to be specified (e.g., the tree structure and partition of data set) for successful implementation of this method. 
	HiGrad also requires modification to the original SGD procedure. 
	The sectioning method, a special case of HiGrad, has the advantage that estimators constructed for different sections are independent. 
	Thus, the asymptotic independence requirement is automatically satisfied.
	However, if we have limited amount of computational budget, focusing on a single long run instead of multiple shorter runs may get us closer to $x^*$ and the normality requirement. 
	
	In Table \ref{tab:main6} and \ref{tab:main8} we compare the finite sample coverage rate of our batch means method and other benchmark methods for logistic regression examples.
	For the batch means method (BM), we set $m=30$ and use IBS for batch size specification.
	When doing joint inference, we set the marginal confidence level at $1-0.05/d$ for BMI based on the Bonferroni correction. 
	
	For HiGrad, we use a two-layer tree structure with $5$ and $6$ nodes for the respective layers. Note that in this case, we have 30 branches in total.
	When doing marginal inference, for BM, we can construct the batch means confidence interval for each parameter (dimension) separately. Algorithm \ref{alg:marg_inf} summarizes our marginal inference procedure.
	\begin{algorithm} 
		\caption{Construct the marginal $100(1-\delta)\%$ confidence interval for each dimension of the model parameter $x^{*}$} \label{alg:marg_inf}
		\begin{algorithmic}[1]
			\STATE \textbf{Input:} The SGD sample path of $\{X_t: 0\leq t\leq T\}$, the number of batches $m$, the relative batch length parameter $w$
			\STATE  Find the appropriate scaling parameter $\alpha_m(\delta,w)$ with $d=1$. 
			\STATE  Calculate the batch means $\Xi_i$ for $i=1,2,\dots,m$
			\STATE For $k=1,\dots, d$, calculate
			\[\bar X_T(k):=\frac{1}{T}\sum_{t=1}^{T}X_t(k),  ~~~ \sigma_{m,T}(k):=\sqrt{\frac{1}{m-1}\sum_{k=1}^{m}(\Xi_i(k)-\bar X_T(k))^2}.\]
			
			\STATE \textbf{Output:} 
			\[R_T(k)=\left[\bar X_T(k)- \sqrt{\frac{\alpha_m(\delta,w)}{m}}\sigma_{m,T}(k) , \bar X_T(k)+ \sqrt{\frac{\alpha_m(\delta,w)}{m}}\sigma_{m,T}(k) \right],\]
			for $k=1,\dots d.$
		\end{algorithmic} 
	\end{algorithm}

	
	In Table \ref{tab:main6}, we show results for confidence regions (joint inference). 
	In Table \ref{tab:main8}, we show results for confidence intervals (marginal inference). The reported coverage
	rate in Table \ref{tab:main8} is the average coverage rate over the $d$ parameters. 
	We observe that BM achieves superior coverage rate comparing to the benchmark methods in all cases. 
	As all the methods we compare are asymptotically valid, we expect all these methods to achieve good coverage rate 
	when the number of iterations (samples) is large enough.
	We also note that the coverage rates deteriorate as the dimension of the problem, $d$, increases. 
	
	\begin{table}[H]
		\caption{Joint coverage rate comparison for different methods: logistic regression} \label{tab:main6}
		\centering
		\begin{tabular}{c|c|c|c|c}\hline
			& $n$ & $4n$ & $7n$ & $10n$  \\ \hline\hline
			\multicolumn{5}{c}{$d=2$}\\ \hline
			BM  & $0.919\pm0.017$ & $0.942\pm0.013$ & $0.936\pm0.015$ & $0.945\pm0.014$  \\ \hline
			BMI  & $0.890\pm 0.019$ & $0.919\pm0.017$ & $0.897\pm0.018$ & $0.899\pm0.018$  \\ \hline
			HiGrad & $0.833\pm0.023$ & $0.879\pm0.020$ & $0.901\pm0.018$ & $0.913\pm0.017$  \\ \hline
			Sectioning  & $0.659\pm0.029$ & $0.807\pm0.024$ & $0.842\pm0.023$ & $0.859\pm0.021$  \\ \hline\hline
			\multicolumn{5}{c}{$d=20$}\\ \hline
			BM & $0.638\pm0.029$ & $0.847\pm0.020$ & $0.878\pm0.020$ & $0.900\pm0.018$  \\ \hline
			BMI  & $0.537\pm 0.030$ & $0.642\pm0.031$ & $0.680\pm0.029$ & $0.698\pm0.028$  \\ \hline
			HiGrad & $0.090\pm0.017$ & $0.427\pm0.030$ & $0.510\pm0.029$ & $0.570\pm0.028$  \\ \hline
			Sectioning  & $0.024\pm0.009$ & $0.226\pm0.026$ & $0.311\pm0.028$ & $0.384\pm0.030$  \\ \hline
		\end{tabular}
	\end{table}
	
	
	\begin{table}[htbp]
		\caption{Marginal coverage rate comparison: logistic regression} \label{tab:main8}
		\centering
		\begin{tabular}{c|c|c|c|c}\hline
			& $n$ & $4n$ & $7n$ & $10n$  \\ \hline\hline
			\multicolumn{5}{c}{$d=2$}\\ \hline
			BM  & $0.938\pm0.015$ & $0.949\pm0.014$ & $0.945\pm0.014$ & $0.953\pm0.013$  \\ \hline
			BMI  & $0.905\pm 0.018$ & $0.920\pm0.017$ & $0.927\pm0.016$ & $0.932\pm0.015$  \\ \hline
			HiGrad & $0.860\pm0.020$ & $0.903\pm0.018$ & $0.913\pm0.017$ & $0.915\pm0.017$  \\ \hline
			Sectioning  & $0.757\pm0.026$ & $0.851\pm0.020$ & $0.872\pm0.020$ & $0.880\pm0.020$  \\ \hline\hline
			\multicolumn{5}{c}{$d=20$}\\ \hline
			BM  & $0.901\pm0.019$ & $0.937\pm0.015$ & $0.945\pm0.014$ & $0.953\pm0.013$  \\ \hline
			BMI  & $0.835\pm 0.023$ & $0.861\pm0.021$ & $0.860\pm0.029$ & $0.866\pm0.021$  \\ \hline
			HiGrad & $0.457\pm0.030$ & $0.610\pm0.029$ & $0.631\pm0.031$ & $0.650\pm0.029$  \\ \hline
			Sectioning  & $0.367\pm0.030$ & $0.535\pm0.031$ & $0.564\pm0.031$ & $0.580\pm0.030$  \\ \hline
		\end{tabular}
	\end{table}
	

	\section{Concluding remarks}
	In this paper, we adapt the batch means method to construct asymptotically valid confidence regions for 
	model parameters in SGD. Our construct is simple and does not require any 
	modification to the underline SGD algorithm. We extend the class of batch means method to allow unequal batch sizes.
	We also extend the asymptotic analysis of Polyak-Ruppert averaging by establishing a process level functional central limit theorem.
	
	Our construction requires that the number of batches $m>d$.
	However, we do not want $m$ to be too large, especially when the sample size $T$ is small.
	Following extensive numerical experiments, we suggest setting $m$ between $20$ and $40$ when $d<10$, and $m$ between $d+5$ and $d+10$ when $d\geq 10$. In terms of the batch size, both ES and IBS work well. 
	Lastly, if we do not have a good knowledge of the starting value for SGD, we would also recommend 
	discard the first few iterations when constructing batches to eliminate the initial transient bias. 
	

	
	
	
	\appendix
	\section{Appendix}
	The proof of Theorem  \ref{th:main} involves two main steps. The first step is to establish the process level convergence of $\bar X_t$.
	The second step is to establish some important properties of the function $g_m$. In particular, we need to show that $g_m(B,w)$ is positive definite with probability 1.
	
	For the first step, we start by presenting two auxiliary lemmas. 
	%
	The first lemma extends the Azuma-Hoeffding inequality to the multidimensional case.
	\begin{lemma}\label{boundmartdiff}
		Let $\M$ be a martingale in $\R^{d}$ with $\M_0=0$, and for every $n$ the martingale difference $\M_n- \M_{n-1}$ satisfies $\|\M_n-\M_{n-1}\|\leq \sigma_n \leq 1/2$. Then for any $a>1$, 
		\[\textnormal{pr}(\|\M_n\|\geq a) \leq 2\exp\left(1-(a-1)^2/(\sum^n_{i=1}2\sigma_i^2)\right).\]
	\end{lemma}
	
		\begin{proof}[Proof of Lemma \ref{boundmartdiff}]
			The lemma follows similar lines of arguments but extends the results of Theorem 1.8 in \cite{ThomaHayes}. 
			
			As $\M$ is a martingale, we have the following decomposition for $\M_i$: $\M_i= \alpha _{i} \M_{i-1} + P_i$, where $\alpha _{i}=(\|\M_{i-1}\|^2)^{-1}\langle \M_i,\M_{i-1} \rangle$ and $P_i$ is orthogonal to $\M_{i-1}$. Define $A_i=(\alpha _{ i}-1)\|\M_{i-1}\|$.
			Then $\|\M_i- \M_{i-1}\|^2 = A_i^2 + \|P_i\|^2 \leq \sigma_i^2$.
			
			We next define $D=(D_i: 1\leq i\leq n)$, $Z=(Z_i: 0\leq i\leq n)$ and $Y=(Y_i: 0\leq i\leq n)$ by the following recursion:
			$Y_0 = 1+ \sum^n_{i=1} \sigma_i^2/(a-1)$ and $Z_0 = 0$. For $i\geq1$,
			\begin{eqnarray*}
				&& D_i = \S(Z_{i-1})\left((Y_{i-1}^2 +2 Y_{i-1} A_i +\sigma_i^2)^{1/2} - Y_{i-1}\right);
				Z_i= Z_{i-1} +D_i;\quad Y_i= Y_0 + |Z_i|;
			\end{eqnarray*}
			where $\S(z) = 1-2 \mathds{1}\left(z<0\right)$.
			
			We first prove that 
			\begin{equation} \label{eq:1d}
			\|M_i\|\leq Y_i.
			\end{equation}
			
			The proof is divided into two steps.
			We first establish a bound for $|Z_i|$. We then prove \eqref{eq:1d} by induction.
			Note that as
			\begin{eqnarray*}
				&&\S(Z_{i-1}) Z_{i} = \S(Z_{i-1}) Z_{i-1} + \S(Z_{i-1})D_i \\
				&=&  |Z_{i-1}|+ (Y^2_{i-1} +2 Y_{i-1}A_i + \sigma_i^2)^{1/2} - Y_{i-1}
				= (Y^2_{i-1} +2 Y_{i-1}A_i + \sigma_i^2)^{1/2} - Y_{0},
			\end{eqnarray*}
			we have
			$|Z_i| =\left|(Y^2_{i-1} +2 Y_{i-1}A_i + \sigma_i^2)^{1/2} - Y_{0} \right|$.
			
			Now by definition $Y_0 >0=\|\M_0\|$. 
			Suppose \eqref{eq:1d}  holds for $i-1$. Then
			\begin{eqnarray*}
				Y_i &=& Y_0 + |Z_i| = Y_0 +\left|(Y^2_{i-1} + 2 Y_{i-1} A_i +(\sigma_i)^2)^{1/2} - Y_0\right|
				\geq (Y_{i-1}^2 +2Y_{i-1} A_i + \sigma_i^2)^{1/2},
			\end{eqnarray*}
			which implies 
			$Y_i^2\geq Y_{i-1}^2 +2Y_{i-1} A_i + \sigma_i^2.$
			
			By the definition of $A_i$, we have
			$\|\M_{i}\|^2 = (\|\M_{i-1}\| + A_i)^2 + \|P_i\|^2.$
			Then
			\[Y_i^2-\|\M_i\|^2 \geq (Y_{i-1} - \|\M_{i-1}\|)(Y_{i-1} + \|\M_{i-1}\|+2A_i) + (\sigma_i^2- A_i^2-\|P_i\|^2)>0,\]
			where the last inequality is due the facts that 
			i) $Y_{i-1} >\|X_{i-1}\|$, 
			ii) $Y_{i-1}\geq Y_0\geq 1 \geq 2\sigma_i \geq 2|A_i|$,
			and iii) $\sigma_i^2\geq A_i^2+\|P_i\|^2$.
			
			We next prove that for $\lambda = (a-1)/(\sum^n_{k=1} \sigma_k^2)$,
			\begin{equation} \label{eq:zbound}
			E[\exp(\lambda Z_n)] \leq \prod_{i=1}^n \cosh (\lambda \sigma_i)\leq \exp\left(\sum_{i=1}^{n}\lambda^2\sigma_i^2/2\right).
			\end{equation}
			The second inequality follows straightforwardly. We shall thus focus on establish the first inequality in \eqref{eq:zbound}.
			To do so, we first establish a bound for $\exp(\lambda D_i)$. 
			In particular, we shall first prove that 
			\begin{equation}\label{eq:dbound}
			\exp(\lambda D_i) \leq \cosh(\lambda \sigma_i) + A_i/\sigma_i \S(Z_{i-1}) \sinh(\lambda \sigma_i). 
			\end{equation}
			
			Fix $Z_{i-1}$, $D_i$ can be view as a function of $A_i$.
			We can thus define $f(x):=\exp(\lambda D_i(x))$.
			Note that $D_i(\sigma_i) = \S(Z_{i-1}) \sigma_i$ and $D_i(-\sigma_i) = -\S(Z_{i-1}) \sigma_i$. Then the line
			linking $(-\sigma_i, f(-\sigma_i))$ and $(\sigma_i, f(\sigma_i))$ takes the form
			$y(x) = \cosh(\lambda \sigma_i) + x/\sigma_i \S(Z_{i-1}) \sinh(\lambda\sigma_i).$ 
			Then, to prove \eqref{eq:dbound}, it suffices to show $\partial^2 f(x)/\partial x^2 > 0$ on the interval $[-\sigma_i, \sigma_i]$.
			\begin{eqnarray*}
				\partial^2 f(x) /\partial x^2 &=& \left( \left(\lambda \partial D_i(x)/\partial x\right)^2 + \lambda \partial^2 D_i(x)/\partial a^2  \right)f(x)\\
				&=& \left(\lambda (Y_{i-1}^2 + 2 Y_{i-1} x+\sigma_i^2 )^{1/2} -\S(Z_{i-1})\right) \frac{\lambda Y_{i-1}^2 f(x)}{(Y^2_{i-1} +2Y_{i-1} x+1)^{3/2}}.
			\end{eqnarray*}
			Now, for $\lambda = (a-1)/\sum^n_{k=1} \sigma_k^2$ and $x\in[-\sigma_i, \sigma_i]$, we have
			\[\lambda (Y_{i-1}^2 + 2 Y_{i-1} A_i +\sigma_i^2 )^{1/2} \geq \lambda(Y_0-\sigma_i) = 1+\lambda (1-\sigma_i) \geq 1 \geq \S(Z_{i-1}).\]
			Thus, $\partial^2 f(x)/ \partial x^2 > 0$ and we have proved \eqref{eq:dbound}. 
			Next, we note that
			$E[\exp(\lambda D_n)|Z_{n-1}] 
			=E[E[\exp(\lambda D_n)|\M_0,\dots,\M_{n-1}]|Z_{n-1}].$ 
			As by \eqref{eq:dbound},
			\begin{eqnarray*}
				&&E[\exp(\lambda D_n)|\M_0,\dots,\M_{n-1}]\\ 
				&\leq& \cosh (\lambda \sigma_n) + \S(Z_{n-1}) \sinh(\lambda \sigma_n) \mathbb{E}[A_n|\M_0,\dots,\M_{n-1}] 
				= \cosh(\lambda \sigma_n),
			\end{eqnarray*}
			we have
			\begin{eqnarray*}
				&&E[\exp(\lambda Z_n)] = E[\exp(\lambda Z_{n-1})\exp(\lambda D_n)]
				= E[\exp(\lambda Z_{n-1}) E[\exp(\lambda D_n)|Z_{n-1})]\\
				&\leq& E[\exp(\lambda Z_{n-1})] \cosh(\lambda \sigma_n)
				\leq\prod_{i=1}^{n} \cosh(\lambda \sigma_i) \mbox{ by recursion.}
			\end{eqnarray*}
			
			Putting \eqref{eq:1d} and \eqref{eq:dbound} together, we have
			\begin{eqnarray*}
				E[\exp(\lambda \|\M_n\|)] &\leq&E [\exp\lambda  Y_n] = \E[\exp(\lambda (Y_0+|Z_n|))]\\ 
				&\leq& e^{\lambda Y_0}\E[\exp(\lambda Z_n)+\exp(-\lambda Z_n)]
				\leq 2e^{\lambda Y_0} \exp\left(\sum_{i=1}^{n}\lambda^2\sigma_i^2/2\right).
			\end{eqnarray*}
			Lastly,
			\begin{eqnarray*}
				\textnormal{pr}(\|\M_n\|\geq a) &\leq& E[\exp(\lambda \|\M_n\| -\lambda a)]\leq  2\exp\left(\lambda Y_0 - \lambda a+ \sum^n_{i=1}\sigma_i^2 \lambda^2/2 \right)\\
				&=& 2 \exp\left(1-(a-1)^2/(2 \sum^n_{k=1} \sigma_k^2)\right) \mbox{ by plugging in the value of $\lambda$}.
			\end{eqnarray*}
		\end{proof}

	The second lemma characterizes the convergence rate of  
	an important term in stochastic gradient descent iterations. It tightens the bound established in \cite{polyak1992acceleration}.
	
	Let 
	\[\bar{\beta}^t_s:=\gamma_s \sum_{i=s}^{t-1}\prod_{k=s+1}^{i}\left(I-\gamma_k \nabla^2 H(x^{*})\right),\] 
	where we define $\prod_{k=s+1}^{s}\left(I-\gamma_k \nabla^2 H(x^{*})\right)
	= I$.
	We also define 
	$\phi^t_s =\bar{\beta}^t_s- \nabla^2 H(x^{*})^{-1}.$
	
	\begin{lemma} \label{lm:sgd_de}
		For $\gamma_t=at^{-r}$ with some $a>0$ and $1/2<r<1$.
		$\sum^{t-1}_{s=0} \|\phi_s^t\| = O(t^{r}). $
	\end{lemma}
		\begin{proof}[Proof of Lemma \ref{lm:sgd_de}]
			We start by summarizing some useful results from \cite{polyak1992acceleration}.
			Let $\beta_s^s=I$ and $\beta_s^{t+1}=\beta_s^t(I-\gamma_t A)$ for $t\geq s$.
			There exists $\lambda,K\in(0,\infty)$ such that for any $s\geq 0$ and $t\geq s$,
			$\|\beta^t_s\|\leq K \exp\left(-\lambda \sum^{t-1}_{i=s} \gamma_i\right),$
			where we define $\sum^{s-1}_{i=s}\gamma_i=0$.
			Now let $S^t_s=\sum^{t-1}_{i=s} (\gamma_s-\gamma_i)\beta^i_s$.
			Then it can be shown that $\phi^t_s=S_s^t - \nabla^2 H(x^{*})^{-1} \beta_s^t$.
			Let $m^i_s =\sum^{i}_{k=s} \gamma_k$. Then
			\[\sum^{t-1}_{i=s} m^i_s \exp(-\lambda m^i_s) = o(1/\gamma_s)
			~~~\mbox{ and } ~~~ \|\bar{\beta}^t_s \|\leq K.\]
			
			We are now ready to prove the lemma. Note that
			$\|\phi^t_s\|\leq \|S^t_s\|+\|\nabla^2 H(x^{*})^{-1}\| \|\beta^t_s\|.$
			
			In what follows, we shall establish bounds for $\|S^t_s\|$ and $\|\beta^t_s\|$ respectively.
			We first note that
			\begin{eqnarray*}
				\|S^t_j\|&=&\left\|\sum^{t-1}_{i=j+1}\left(\sum^{i}_{k=j+1} (\gamma_{k-1}-\gamma_{k})\right)\beta^i_j\right\|\\
				&=&\left\|\sum^{t-1}_{i=j+1}\left(\sum^{i}_{k=j+1} (\gamma_{k-1}-\gamma_{k})\gamma_{k-1}(\gamma_{k-1})^{-1}\right)\beta^i_j\right\|\\
				&\leq& K(\gamma_{j}-\gamma_{j+1})(\gamma_j)^{-1} \sum^{t-1}_{i=j} m^i_j \exp(-\lambda m^i_j).
			\end{eqnarray*}  
			Thus, for $j$ large enough, $\|S^t_j\| \leq  K(\gamma_{j}-\gamma_{j+1})/\gamma_j^2$.
			We also notice tha by L'Hospital's Rule, $(\gamma_{j}-\gamma_{j+1})(\gamma_j^2)^{-1} = O(j^{-(1-r)})$.
			Then for $t$ large enough,
			\begin{equation} \label{eq:s}
			\sum^{t-1}_{j=0} \| S^t_j\|\leq K \sum^{t-1}_{j=0} j^{-(1-r)} \leq K \int^t_{0} x^{-(1-r)} =O(t^{r}).
			\end{equation}
			
			Next, we note that
			\begin{eqnarray}\label{eq:beta_sum_bd}
			\sum^{t-1}_{j=0} \|\beta^t_j\| \leq \sum^{t-1}_{j=0} \exp(-\lambda (t-j)\gamma_t)
			\leq  \frac{1}{1-\exp(-\lambda \gamma_t)} = O(\gamma_t^{-1}) = O(t^{r}). 
			\end{eqnarray}
			
			Combining \eqref{eq:s} and \eqref{eq:beta_sum_bd}, we have
			\[\sum^{t-1}_{j=0} \|\phi^t_j\| \leq  \sum^{t}_{j=0} \|S^t_j\|+\left\|\nabla^2 H(x^{*})^{-1}\right\| \cdot \sum^{t}_{j=0} \|\beta^t_j\|=O(t^{r}).\]
		\end{proof}
		
	Next, we establish the process level convergence of $\bar X_t$
	\begin{theorem} \label{th:fcll} 
		Under Assumption \ref{ass:main1} and \ref{ass:main2}, there exists a matrix $G$, such that
		\[n^{1/2} t(\bar X_{nt}-x^*)\Rightarrow GB(t) \mbox{ in $D(0,\infty)$ as $n\rightarrow\infty$,}\]
		where $D(0,\infty)$ denotes the space of right continuous functions with left limit endowed with Skorokhod $J_1$ topology.
	\end{theorem}
	
	\begin{proof}[Proof of Theorem \ref{th:fcll}]
		We start by summarizing some useful results from \cite{polyak1992acceleration}.
		We first note that $\bar X_t$ has the following decomposition: 
		\[\bar X_t-x^*= J^{(1)}(t)+J^{(2)}(t)+J^{(3)}(t),\]
		where 
		\begin{eqnarray*}
			J^{(1)}(t) &=& -t^{-1} \sum^{t-1}_{s=0}(\nabla^2 H(x^{*}) + \phi^t_s) (\nabla H(X_s) - \nabla^2 H(x^{*}) \Delta_s,\\
			J^{(2)}(t) &=& t^{-1}\sum^{t-1}_{s=0} \nabla^2 H(x^{*})^{-1} \xi_s, \quad
			J^{(3)}(t) = t^{-1}\sum^{t-1}_{s=0} \phi^t_s \xi_s.
		\end{eqnarray*}
		Recall that $\xi_t=\mathcal{G}(X_{t-1},\zeta_t)-\nabla H(X_{t-1})$ and $\Delta_t=X_t-x^*$.
		We also have the following properties about the decomposition: 
		\begin{itemize}
			\item[P1)] $ t^{-1/2} \sum^{t-1}_{i=1} \|\Delta_i\|^2 \to 0$ almost surely (a.s.) as $t\to \infty$, 
			\item[P2)] $\|\phi^t_s\| \leq K$ for some $K\in (0,\infty)$ 
			\item[P3)] $\sum_{s=1}^{t}\|\phi_s^t\|=O(t^{r})$. 
		\end{itemize}
		We comment that P3 is not provided in \cite{polyak1992acceleration}. We establish it in Lemma \ref{lm:sgd_de}. 
		
		We are now ready to establish the functional level convergence results for each part of the decomposition.
		For  $J^{(1)}$, we have 
		\begin{eqnarray*}
			&&\sup_{0\leq t\leq T}\|tn^{1/2}J^{(1)}(nt)\|\\
			&\leq& \sup_{0\leq t\leq T}  n^{-1/2} \sum^{nt-1}_{s=1}\|(\nabla^2 H(x^{*}) + \phi^t_s) (\nabla H(X_s) - \nabla^2 H(x^{*}) \Delta_s)\| \\
			&\leq& (\|\nabla^2 H(x^{*})\| + K) \frac{C}{2} \sup_{0\leq t\leq T} n^{-1/2} \sum^{nt-1}_{i=1} \|\Delta_i\|^2 \mbox{ by P2 and Assumption \ref{ass:main1} }\\
			&\leq& \frac{C}{2}(\|\nabla^2 H(x^{*})\| + K)\T^{1/2}((nT)^{1/2})^{-1} \sum^{nT-1}_{i=1} \|\Delta_i\|^2 \to 0 \mbox{ a.s. as $n\rightarrow\infty$ by P1.}
		\end{eqnarray*}
		Thus, $tn^{1/2}J^{(1)}(nt)\Rightarrow 0$ in $D(0,\infty)$ as $n\rightarrow \infty$.
		
		For $J^{(2)}$,  
		let 
		$\M_n(t):=n^{-1/2}\sum_{s=1}^{nt} \xi_s.$
		We next establish the functional central limit theorem (FCLT) for $M_n$, i.e. there exists a matrix $U$ such that
		\begin{equation} \label{eq:fcltm}
		\M_n(t)\Rightarrow UB(t) \mbox{ in $D(0,\infty)$ as $n\rightarrow\infty$.}
		\end{equation}
		Under Assumption \ref{ass:main2}, $\xi_t$'s form a Martingale-difference sequence.
		Following Theorem 8.1 in \cite{whittproofs}, we only need to verify the following two conditions:\\
		\begin{itemize}
			\item[C1)] For each $t>0$, $\lim_{n \to \infty} E [J(\M_n, t)] =0,$
			where $J$ is the maximum jump function, i.e. $J(x,t):= \sup \{ \|x(s)-x(s-)\| : 0<s\leq t\}$.
			\item[C2)] For each $(i,j)$, $1\leq i,j \leq d$, there exits a constant $U_{ij}$, such that 
			$\M_{n,i}, \M_{n,j}](t)\Rightarrow U_{ij} t$ as $n\rightarrow \infty$, 
			where $M_{n,i}$ denotes $i$-th entry of $\M_n$, and $[\M_{n,i}, \M_{n,j}]$ is the square-bracket process.
		\end{itemize}
		For C1), under the boundedness condition of the Martingale differences (Assumption \ref{ass:main2}), 
		\[ E [J(\M_{n}, T)] =
		E \left[n^{-1/2} \sup_{0< s\leq Tn} \|\xi_s\|\right]
		\leq \frac{M}{n} \rightarrow 0 \mbox{ as $n\rightarrow \infty$.}\]
		
		For C2), we have	
		\[\begin{split}
		[\M_{n i}, \M_{n j}] (t) 
		&= t/(nt) \sum^{nt}_{s=1} \xi_{si} \xi_{sj}\\
		&= \underbrace{t/(nt) \sum^{nt}_{s=1} (\xi_{si}\xi_{sj}- E[\xi_{si} \xi_{sj}|\mathcal{F}_{s-1}])}_{(a)}+\underbrace{t/(nt) \sum^{tn}_{s=1} E[\xi_{si} \xi_{sj}|\mathcal{F}_{s-1}]}_{(b)}.
		\end{split}\]
		
		Under Assumption \ref{ass:main2}, (a) is again a martingale.
		We can thus apply martingale law of large number \cite{csorgHo1968strong}, i.e.,
		$ (nt)^{-1} \sum^{nt}_{s=1} (\xi_{si}\xi_{sj}- E[\xi_{si} \xi_{sj}])\Rightarrow0$  as $n\rightarrow\infty$. 
		For (b), under Assumption \ref{ass:main2}, we have
		$t/(nt)\sum^{tn}_{s=1} E[\xi_{si} \xi_{sj}|\mathcal{F}_{s-1}] \Rightarrow U_{ij}t$ as $n\rightarrow\infty$.

		Based on \eqref{eq:fcltm}, we have for $G=\nabla^2 H(x^{*})^{-1}U$,	
		\[ tn^{1/2}J^{(2)}(nt)  \Rightarrow G B(t) \mbox{in $D(0,\infty)$ as $n\to \infty$}.\]	
		
		For $J^{(3)}$,  by Assumption \ref{ass:main2}, we have for any $\delta>0$ and $n$ large enough,
		\begin{eqnarray*}
			&&\textnormal{pr}\left(\sup_{1\leq t \leq nT} \left\| n^{-1/2} \sum^t_{i=1} \phi^t_i \xi_i \right\| \geq \delta\right)\\
			&=&\textnormal{pr}\left(\sup_{1\leq t \leq nT} \left\| (2MK)^{-1}\sum^t_{i=1} \phi^t_i \xi_i \right\| \geq n^{1/2}\delta/(2MK)\right)\\
			&\leq& \sum^{nT}_{t=1} 2\exp\left(1-(n^{1/2}\delta/(2MK) -1)^2/(2 \sum^t_{s=1} M^2/(2MK)^2\|\phi^t_s\|^2)\right) \mbox{ by Lemma \ref{boundmartdiff}}\\
			&\leq& \sum^{nT}_{t=1} \exp\left(1-2K(\delta/(2MK) -n^{-1/2})^2/(n^{-1}\sum^t_{s=1}\|\phi^t_s\|)\right)\\
			&\leq& \sum^{nT}_{t=1} 2\exp\left(1-2K(\delta/(2MK) -n^{-1/2})^2/(n^{-1}C^{\prime} t^{r})\right) \mbox{ for some $C^{\prime}>0$ by P3}\\	
			&\leq& 2nT\exp\left(1-2K(\delta/(2MK) -n^{-1/2})^2/(C^{\prime} T^{r}) n^{1-r}\right)					
			\rightarrow 0 \mbox{ as $n\rightarrow \infty$.} 
		\end{eqnarray*}
		Thus, $tn^{1/2}J^{(3)}(nt)\Rightarrow 0$ in $D(0,\infty)$ as $n\rightarrow \infty$.	
	\end{proof}
	We note from Theorem \ref{th:fcll} that if we fix $t=1$, then we have $n^{1/2} (\bar X_{n}-x^*)\Rightarrow N(0,G)$ as $n\to\infty$,
	i.e., the FCLT result we established is stronger than the large sample central limit theorem.
	We also comment that FCLT is required for batch means and a more general class of cancellation methods known as the standardized
	time series \cite{Glynn:1990}.

	We next carry out the second step. For the batch means method to be valid, we require that the number of batches $m\geq d+1$. 
	This is because when $m\leq d$, the estimated covariance matrix, $S_m(T)$, is likely to be degenerate. Specifically, 
	from Theorem \ref{th:fcll}, we have that for any $m\in \Z^+$,
	$TS_m(T) \Rightarrow Gg_m(B,w)G^T \mbox{ as $T\rightarrow\infty$}.$
	The following lemma characterizes the behavior of $g_m(B,w)$ for different values of $m$, including when $m\leq d$.
	\begin{lemma} \label{lm:h}
		For $w>0$, when $m\geq d+1$, $g_m(B,w)$ is positive definite with probability 1;
		when $m\leq d$, $g_m(B,w)$ is degenerate with probability 1. 
	\end{lemma} 
		\begin{proof}[Proof of Lemma \ref{lm:h}]
			Recall $w_i = c_i-c_{i-1}$, Let $N_i = w_i^{-1/2}B(c_i)-B(c_{i-1})$, and
			$r_i = \left[-w_1^{1/2},\dots, w_i^{-1/2}-w_i^{1/2}, \dots, -w_m^{1/2}\right]^{T}$,
			for $i=1,2,\dots, m$. Then
			\[g_m(B,w) =(m-1)^{-1} N \left(\sum^m_{i=1} r_i r_i^T\right) N ^T = (m-1)^{-1} N V N ^T,\] 
			where  $N = [N_1, \dots, N_m]$, is a $d\times m $ matrix whose columns are independent and identically distributed $d$-dimensional standard Gaussian random vectors,
			and $V$ is an $m\times m$ matrix with $V_{ii} = 1/w_i-2 + m w_i$ and $V_{ij} = -(w_i/w_j)^{1/2}-(w_j/w_i)^{1/2} + m (w_i w_j)^{1/2}$ for $i\neq j$.
			
			In what follows, we shall prove that $V$ has rank $m-1$. 
			We first note that as
			$\sum^m_{i=1} w_i^{1/2} V_i = 0$, 
			where $V_i$ denotes $i$-th row of $V$,
			$\mbox{rank}(V) \leq m-1$.
			We next look at the `upper-left corner' $(m-1) \times (m-1)$ sub-matrix of $V$, which we denoted as $\tilde V$. 
			We can decomposition $\tilde{V}$ as
			$\tilde{V} = \tilde{V}^1 + \Delta,$
			where
			$ \tilde{V}^1_{ij} = \tilde{V}_{ij}$ for $i\neq j$, and  $ \tilde{V}^1_{ii} = (m-1)/ (m w_i) -2 +m w_i$;
			$\Delta_{ij}=0$ for $i\neq j$, and $\Delta_{ii} = (m w_i)^{-1}>0$. 
			
			Let $\tilde{w}_i = (m-1)w_i/m$, 
			$\tilde{r}_i = \left[-\tilde{w}_1^{1/2},\dots, (\tilde{w}_i)^{-1/2}-(\tilde{w}_i)^{1/2}, \dots, -(\tilde{w}_{m-1})^{1/2}\right]^T$,
			for $i=1,\dots, m-1$. Then we have $\tilde{V}^1 = \sum^{m-1}_{i=1} \tilde{r}_i \tilde{r}_i^T$. This suggests $\tilde{V}_1$ is positive semi-definite.  
			As $\Delta$ is strictly positive definite, $\tilde{V}$ is positive definite. This indicates that $\mbox{rank}(V)\geq m-1$. 
			Thus, $\mbox{rank}(V) = m-1$.

			Now for $m\leq d$, 
			$ \mbox{rank}(g_m(B, w)) \leq \mbox{rank}(V)\leq m-1 < d,$
			Thus, $g_m(B,w)$ is degenerate.
			
			For $m>d$, 
			define $\mathcal{P}(N)=\det(NN^T)$, which is a  polynomial function over entries of $N$. 
			Notice as all entries of $N$ are independent and identically distributed standard Normal random variables, 
			$\{X\in \mathbb{R}^{d \times m} :\mathcal{P}(X) =0\}$ has Lebesgue measure $0$.
			Therefore, $\textnormal{pr}(\det(NN^T) = 0) =0$. 
			This indicates that $N$ has rank $d$ a.s.. 
			Since $V$ is of rank $m-1$ and positive semi-definite, we can decompose it as 
			$V= P\Lambda P^T,$
			where $\Lambda$ is a diagonal matrix with $\Lambda_{ii}>0$ for $i=1,2, \dots, m-1$ and $\Lambda_{mm}=0$, 
			$P$ is an orthogonal matrix. 
			Next, note that $NP\Lambda^{1/2} \overset{d}{=} N\Lambda^{1/2} = [\tilde{N}, 0]$ where 
			$\tilde{N} =[\lambda_1N_1, \dots, \lambda_{m-1}N_{m-1}]$.
			It is easy to see that $\tilde N$ is again a gaussian random matrix with each independent and identically distributed standard Normal elements.
			Then $\mbox{rank}(\tilde{N}) = d$ a.s..
			Therefore, 
			$g_m(B, w)=(m-1)^{-1} N V N ^T $ has the same distribution as  $(m-1)^{-1}\tilde{N} \tilde{N}^T$, having rank $d$ almost surely, i.e., $g_m(B, w)$ is positive definite almost surely.	
		\end{proof}
		
	Now we are ready to prove 
	Theorem \ref{th:main}.
	\begin{proof}[Proof of Theorem \ref{th:main}]
		The proof builds on verifying the conditions for Theorem 1 in \cite{MunozG01}.
		We denote $B$ as a $d$-dimensional Brownian motion.
		We first show that $g_m(x, c)$ satisfies following four properties:
		\begin{enumerate}[label=(\alph*)]
			\item $g_m(Gx, w) = G g_m(x, w) G^T$ for any non-singular $d\times d $ matrix G. 
			\item $g_m(x-\beta \eta, w) = g_m(x, w)$ for $x\in C[0,1]^d$ and $\beta\in R^d$, where $\eta(t) := t$, $0\leq t \leq 1$.
			\item $g_m(B, w)$ is positive definite and symmetric almost surely.
			\item $\textnormal{pr}(B\in D(g_m(\cdot,w))) =0 $ where $D(g_m(\cdot,w))$ is the set of discontinuities of $g_m(\cdot,c)$.
		\end{enumerate}
		For (a), we note that
		\begin{eqnarray*}
			g_m(Gx,w) &=&\frac{1}{m-1} \sum_{i=1}^m \left(\frac{Gx(c_i)-Gx(c_{i-1})}{c_i- c_{i-1}} -Gx(1)\right)\left(\frac{Gx(c_i)-Gx(c_{i-1})}{c_i- c_{i-1}} -Gx(1)\right)^T\\
			&=&	\frac{G}{m-1} \sum_{i=1}^m \left(\frac{x(c_i)-x(c_{i-1})}{c_i- c_{i-1}} -x(1)\right)\left(\frac{x(c_i)-x(c_{i-1})}{c_i- c_{i-1}} -x(1)\right)^T G^T\\
			&=& G g_m(x,w) G^T.
		\end{eqnarray*}
		
		For (b), we have
		\begin{eqnarray*}
			g_m(x-\beta J,w) &=&\frac{1}{m-1} \sum_{i=1}^m \left(\frac{(x-\beta J)(c_i)-(x-\beta J)(c_{i-1})}{c_i- c_{i-1}} -(x-\beta J)(1)\right) \\
			&&\left(\frac{(x-\beta J)(c_i)-(x-\beta J)(c_{i-1})}{c_i- c_{i-1}} -(x-\beta J)(1)\right)^T\\
			&=&	\frac{1}{m-1} \sum_{i=1}^m \left(\frac{x(c_i)-x(c_{i-1})}{c_i- c_{i-1}} -x(1)\right) \left(\frac{x(c_i)-x(c_{i-1})}{c_i- c_{i-1}} -x(1)\right)^T \\
			&=& g_m(x,w).
		\end{eqnarray*}
		
		Lastly, (c) follows from Lemma \ref{lm:h}. Since $g_m(\cdot,w)$ is continuous on $C[0,1]^d$, (d) is also satisfied. 
		
		Let $\bar{Y}_T(u) = T^{-1} \sum^{uT}_{i=1} X_i$, $0\leq u \leq 1$. Note that
		$ S_m(T) = g_m(\bar{Y}_T,w)$.
		From Theorem \ref{th:fcll}, $\bar{Y}_T(u)\Rightarrow GB(t)$ in $D[0,1]$ as $T\rightarrow\infty$.
		Then, from Theorem 1 in \cite{MunozG01}, we have 
		\[\begin{split}
		&\Gamma_T=m(m-d)/(d(m-1))(\bar X_T - x^*)^TS_m^{-1}(T)(\bar X_T - x^*)\\
		\Rightarrow& m(m-d)/(d(m-1)) B^T(1) g_m(B,w)^{-1} B(1) \mbox{ as $T\rightarrow\infty$.}\end{split}\]
		Moreover, we note that
		\[\frac{B(c_i)-B(c_{i-1})}{c_i- c_{i-1}} -B(1) = \frac{1}{c_i-c_{i-1}}\left( B(c_i) - c_i B(1) -(B(c_{i-1})-c_{i-1}B(1))\right).\]
		As $B(u)-uB(1)$, $0\leq u\leq 1$, is independent of $B(1)$, $g_m(B,w)$ independent of $B(1)$.
	\end{proof}

	\bibliographystyle{plain}
	\bibliography{batch_mean_ref}

\begin{thebibliography}{10}

\bibitem{agarwal2009information}
A.~Agarwal, M.~J. Wainwright, P.~L. Bartlett, and P.K. Ravikumar.
\newblock Information-theoretic lower bounds on the oracle complexity of convex
  optimization.
\newblock In {\em Advances in Neural Information Processing Systems}, pages
  1--9, 2009.

\bibitem{Chen:2018}
X.~Chen, J.~D. Lee, X.~T. Tong, and Y~Zhang.
\newblock Statistical inference for model parameters in stochastic gradient
  descent.
\newblock https://arxiv.org/pdf/1610.08637.pdf, 2018.

\bibitem{csorgHo1968strong}
Mikl{\'o}s Cs{\"o}rg{\H{o}}.
\newblock On the strong law of large numbers and the central limit theorem for
  martingales.
\newblock {\em Transactions of the American Mathematical Society},
  131(1):259--275, 1968.

\bibitem{Fang:2017}
Y.~Fang, J.~Xu, and L.~Yang.
\newblock On scalable inference with stochastic gradient descent.
\newblock arXiv preprint arXiv:1707.00192, 2017.

\bibitem{Glynn:1990}
P.W. Glynn and D.L. Iglehart.
\newblock Simulation output analysis using standardized time series.
\newblock {\em Mathematics of Operations Research}, 15(1):1--16, 1990.

\bibitem{ThomaHayes}
T.P. Hayes.
\newblock A large-deviation inequality for vector-valued martingales.
\newblock available at https://www.cs.unm.edu/~hayes/papers/VectorAzuma/, 2005.

\bibitem{Glynn:2002}
M.~Hsieh and P.W. Glynn.
\newblock Confidence region for stochastic approximation algorithms.
\newblock In {\em Winter Simulation Conference}, 2002.

\bibitem{Kingma:2015}
D.P. Kingma and J.L. Ba.
\newblock {ADAM}: A method for stochastic optimization.
\newblock In {\em ICLR}, 2015.

\bibitem{MunozG01}
D.F. Munoz and P.W. Glynn.
\newblock Multivaraite standardized time series for steady-state simulation
  output analysis.
\newblock {\em Operations Research}, 49(3):413--422, 2001.

\bibitem{Nemirovski:2009}
A.~Nemirovski, A.~Juditsky, G.~Lan, and A.~Shapiro.
\newblock Robust stochastic approximation approach to stochastic programming.
\newblock {\em SIAM Journal on Optimization}, 19(4):1574--1609, 2009.

\bibitem{whittproofs}
G.~Pang, R.~Talreja, and W.~Whitt.
\newblock Martingale proofs of many-server heavy-traffic limits for markovian
  queues.
\newblock {\em Probability Surveys}, 4:193--267, 2007.

\bibitem{polyak1992acceleration}
B.~T. Polyak and A.~B. Juditsky.
\newblock Acceleration of stochastic approximation by averaging.
\newblock {\em SIAM Journal on Control and Optimization}, 30(4):838--855, 1992.

\bibitem{RobbinMonros:1951}
H.~Robbins and S.~Monro.
\newblock A stochastic approximation method.
\newblock {\em The Annals of Mathematical Statistics}, 22(3):400--407, 1951.

\bibitem{schruben1983confidence}
L.~Schruben.
\newblock Confidence interval estimation using standardized time series.
\newblock {\em Operations Research}, 31(6):1090--1108, 1983.

\bibitem{Su:2018}
W.~J. Su and Y.~Zhu.
\newblock Uncertainty quantification for online learning and stochastic
  approximation via hierarchical incremental gradient descent.
\newblock arXiv preprint arXiv:1802.04876, 2018.

\bibitem{Toulis:2017}
P.~Toulis and E.~M. Airoldi.
\newblock Asymptotic and finite-sample properties of estimators based on
  stochastic gradients.
\newblock {\em The Annals of Statistics}, 45(4):1694--1727, 2017.

\bibitem{zhang2004solving}
T.~Zhang.
\newblock Solving large scale linear prediction problems using stochastic
  gradient descent algorithms.
\newblock In {\em Proceedings of the twenty-first international conference on
  Machine learning}, page 116. ACM, 2004.

\end{thebibliography}

\end{document}